\documentclass[letterpaper]{amsart}

\usepackage{graphicx}
\usepackage{amssymb}
\usepackage{amsthm}
\usepackage{amsmath}
\usepackage{algorithm}
\usepackage{algorithmic}

\newcommand{\R}{\mathbb{R}}     

\newcommand{\Cl}{\mathrm{C}\ell}
\DeclareMathOperator{\Spin}{Spin}
\DeclareMathOperator{\Pin}{Pin}
\newcommand{\e}{\mathbf{e}}

\DeclareMathOperator{\N}{n}
\DeclareMathOperator{\sgn}{sgn}

\allowdisplaybreaks

\newtheorem{theorem}{Theorem}

\newtheorem{lemma}{Lemma}

\newtheorem{proposition}{Proposition}

\begin{document}

\title{Swing-twist decomposition in Clifford algebra}

\author{Przemys{\l}aw Dobrowolski}

\address{Faculty of Mathematics and Information Science \\
         Warsaw University of Technology, Poland}

\begin{abstract}
The swing-twist decomposition is a standard routine in motion planning for humanoid limbs. In this paper the decomposition formulas are derived and discussed in terms of Clifford algebra. With the decomposition one can express an arbitrary spinor as a product of a twist-free spinor and a swing-free spinor (or vice-versa) in 3-dimensional Euclidean space. It is shown that in the derived decomposition formula the twist factor is a generalized projection of a spinor onto a vector in Clifford algebra. As a practical application of the introduced theory an optimized decomposition algorithm is proposed. It favourably compares to existing swing-twist decomposition implementations.
\end{abstract}

\maketitle

\section{Introduction}

Swing-twist decomposition of rotations is commonly used in context of humanoid motion planning. Consider movement of an arm reaching some predefined position. In order to displace it properly, a controller calculates twist factor of the corresponding rotation. Having this factor computed, a controller is then able to apply some corrections to the motion so that the arm is not unnaturally twisted during the motion. \textbf{Swing-twist decomposition} is an inherent part of a correction algorithm. It allows one to decompose an arbitrary rotation into a swing part (tilt of a given axis) and a twist part (rotation around a given axis). In this paper the decomposition is derived and discussed in terms of Clifford algebra.

Swing-twist decomposition has already been considered for a few decades. Different authors have obtained equivalent formulas, in particular for quaternion algebra. Current literature tends to neglect a deeper consideration on the spin-twist decomposition. There is wide range of publications on humanoid motion planning. Most of these works relate to swing-twist decomposition in some way. In the paper, the most revelant approaches to swing-twist decomposition are quoted and compared. Starting from the recent PhD dissertation by Huyghe (\cite{huyghe}, 2011), the swing-twist decomposition is introduced in quaterion algebra by a proposed therein projection operator. Unfortunately, the origin of the projection operator is not enough explained. In two previous papers, Baerlocher (\cite{baerlocher}, 2001) and Baerlocher, Boulic (\cite{baerlocher_boulic}, 2000) investigate joint boundaries for ball-and-socket joints using swing-twist decomposition. Grassia (\cite{grassia}, 1998) discusses features of a swing-twist decomposition in terms of \textbf{an exponential map}. However, the author does not present any related formula for the decomposition. A classic work by Korein (\cite{korein}, 1984) contains most of the initial results on body positioning and joint motion. Among others the work uses swing-twist decomposition of rotations to constrain movement of an elbow.

None of these are generalized to Clifford or geometric algebras which has deeper consequences than previous results. In this paper the swing-twist decomposition is derived as an inverse of a formula expressing the set of rotations which move a given initial vector to cover a given terminal one.

\section{Existing and related solutions}
\label{section_existing_and_related_solutions}

Swing-twist decomposition splits a given rotation in two parts: a swing part and a twist part. A schematic view of a limb rotating with a ball joint is presented in figure \ref{img:joint}.

\begin{figure}[ht]
\centering
\includegraphics[width=8cm]{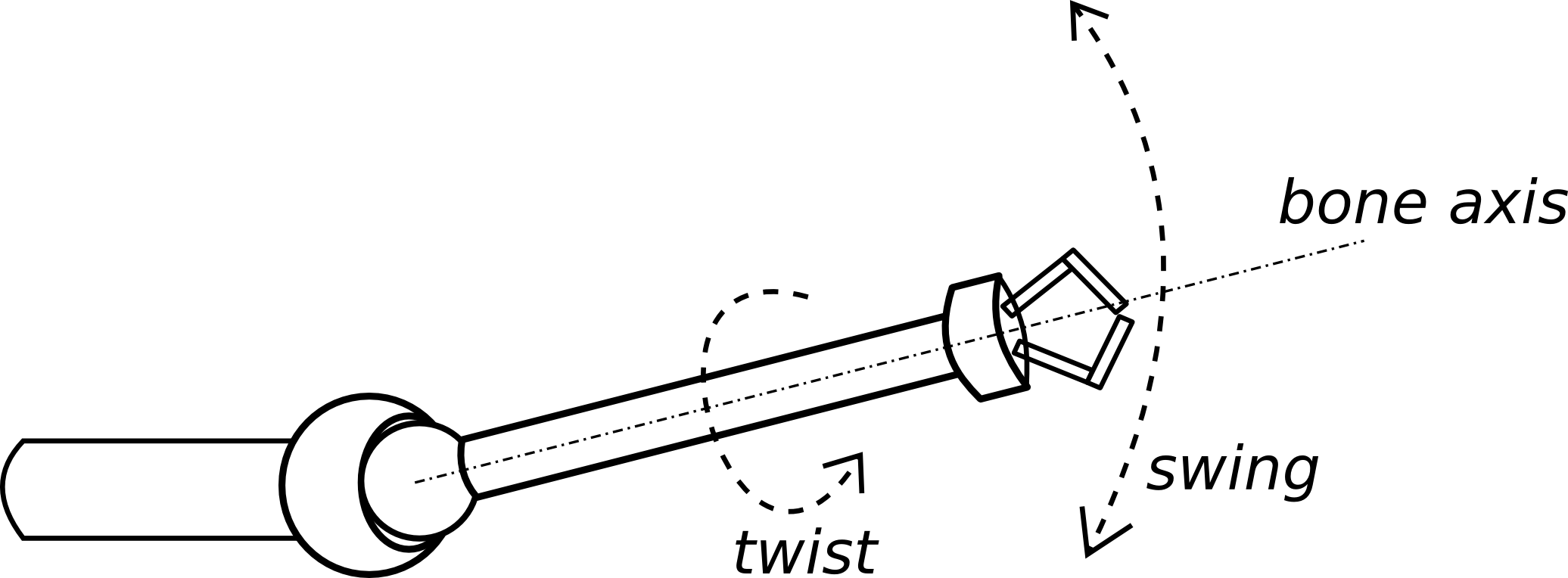}
\caption{A limb with a ball joint}
\label{img:joint}
\end{figure}

By \textbf{zero twist reference vector} one refers to a base vector with respect to which the swing-twist decomposition is performed. Usually, this is the bone of a rotating limb.

There are several existing approaches. These solutions differ in terms of performance, complexity and result exactness.

\subsection{Direct method}

For quaternions one can make the following argument. Let $q \in \mathbb{H}$ be a quaternion and $v$ be a zero-twist reference vector. In case of twist-after-swing type of decomposition (see section \ref{sec:two_st_decompositions}) the initial vector becomes the given $v$ vector and the terminal vector is $w = q v q^*$. Schematic view is presented in figure \ref{img:direct_decomposition}.

\begin{figure}[ht]
\centering
\includegraphics[width=3cm]{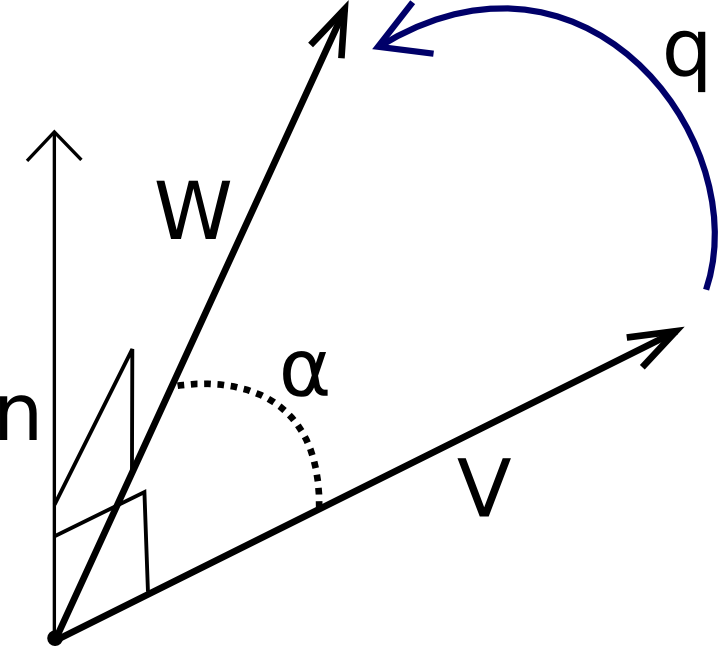}
\caption{A direct decomposition of a quaternion}
\label{img:direct_decomposition}
\end{figure}

Swing quaternion can be calculated with an axis angle representation of the quaternion. Here, the axis is a normalized vector perpendicular to $v$ and $w$ and the angle is equal to the angle between $v$ and $w$ so:
\begin{align*}
& n = \frac{v \times w}{\|v \times w\|} \\
& \cos(\alpha) = v \cdot w \\
& q_s = \cos(\alpha / 2) + \sin(\alpha / 2) (n_x \mathbf{i} + n_y \mathbf{j} + n_z \mathbf{k})
\end{align*}

Twist quaternion is calculated by the following inversion formula:
\begin{align*}
& q = q_t q_s \Longrightarrow q_t = q q_s^{-1} \\
& q_t = q (\cos(\alpha / 2) - \sin(\alpha / 2) (n_x \mathbf{i} + n_y \mathbf{j} + n_z \mathbf{k}))
\end{align*}

In this formulation both trigonometric and inverse trigonometric functions are used. It is a practical disadvantage since it requires computationally expensive functions which can also face some accuracy issues.

\subsection{Huyghe's method}

Huyghe uses swing-after-twist decomposition type. First a simplified formula is derived which is a decomposition of an arbitrary quaternion $q$ into a product
\begin{equation*}
q = q_s q_t
\end{equation*}
with respect to a constant Z-axis aligned zero-twist reference vector (called there "a twist axis"). With the following coordinates
\begin{align*}
q & = w + x \textbf{i} + y \textbf{j} + z \textbf{k} \\
q_t & = w_t + z_t \textbf{k} \\
q_s & = w_s + x_s \textbf{i} + y_s \textbf{j}
\end{align*}
the author shows using some algebraic transformations that the decomposition is
\begin{align*}
w_t & = \frac{\pm w}{\sqrt{w^2 + z^2}} \\
z_t & = \frac{\pm z}{\sqrt{w^2 + z^2}} \\
w_s & = w_t w + z_t z \\
x_s & = w_t x - z_t y \\
y_s & = w_t y + z_t x
\end{align*}
It must be noted that some of the results were not rigorously stated. It particular, scenarios when some coefficients are equal to zero are not discussed separately which leads to possible division by zero.

In the second part of the thesis, Huyghe derives a generalized formula for swing-twist decomposition. Assume that axis-angle representation of quaternions is
\begin{align*}
q_s = [w_s, v_s] = [\cos(\frac{\sigma}{2}), u_s \sin(\frac{\sigma}{2})] \\
q_t = [w_t, v_t] = [\cos(\frac{\tau}{2}), u_t \sin(\frac{\tau}{2})]
\end{align*}
Multiplying both quaternions one writes
\begin{equation} \label{eqn:huyghe_formula_one}
q = q_s q_t = [ w_s w_t - v_t \cdot v_s, w_s v_t + w_t v_s + v_s \times v_t ]
\end{equation}
Huyghe notes that the axes $v_s$ and $v_t$ are perpendicular so formula \eqref{eqn:huyghe_formula_one} simplifies to
\begin{equation} \label{eqn:huyghe_formula_two}
q = [ w_s w_t, w_s v_t + w_t v_s + v_s \times v_t ]
\end{equation}
Next, the author introduces a new quaternion $q_p$ (no origin is provided) which is "a projected version of the initial quaternion $q$ onto the twist axis" (\cite{huyghe})
\begin{align*}
q_p & = [w, (v \cdot u_t) u_t] \\
& = [w_s w_t, (w_s v_t \cdot u_t + w_t v_s \cdot u_t + (v_s \times v_t) \cdot u_t ) u_t ] \\
& = [w_s w_t, w_s \|u_t\|^2 \sin(\frac{\tau}{2}) u_t] \\
& = [w_s w_t, w_s v_t]
\end{align*}
Normalization of quaternion $q_p$ gives
\begin{equation*}
\frac{q_p}{\|q_p\|} = \frac{[w_s w_t, w_s v_t]}{\sqrt{w_s^2 w_t^2 + w_s^2 \|v_t\|^2}} = \frac{w_s [w_t, v_t]}{w_s \sqrt{w_t^2 + \|v_t\|^2}} = q_t
\end{equation*}
which in result turns out to be twist quaternion. Note that in the above equation Huyghe does not consider $w_s = 0$ which is a drawback. Remaining swing quaternion is calculated from
\begin{equation*}
q_s = q q_t^*
\end{equation*}
which completes the decomposition.

\section{Preliminaries}

The presented results relate to $\Cl_3 := \Cl_{3,0}(\R)$ - Clifford algebra of 3-dimensional real space. Herein, a spinor is given by $s = a + b \e_{12} + c \e_{23} + d \e_{31} \in \Spin(3)$ and a vector is given by $v = x \e_0 + y \e_1 + z \e_2$. A rotation of a vector is given by Clifford multiplication:
\begin{equation*}
v' = s v s^{-1}
\end{equation*}
for an arbitrary vector $v$ and a spinor $s$.

In this paper ${\square}^{-1}$ denotes the inverse of a spinor. In case of $\Spin(3)$ it is equivalent to conjugation:
\begin{equation*}
s^{-1} = \tilde{s} = a - b \e_{12} - c \e_{23} - d \e_{31}
\end{equation*}
A spinor can be written as a sum of its scalar and bivector parts:
\begin{equation*}
s = [s]_0 + [s]_2
\end{equation*}
where $\quad [s]_0 \in \bigwedge_0 \R^3$ and $[s]_2 \in \bigwedge_2 \R^3$. In several places the Hodge star $\star$ operator is used. In an orthonormal basis it is defined as
\begin{equation*}
\star (\e_1 \wedge \e_2 \wedge \cdots \wedge \e_k) = \e_{k+1} \wedge \e_{k+2} \wedge \cdots \wedge \e_n
\end{equation*}
but in this paper it is used only in relation to the bivector part of a spinor. In this case the formula can be simplified to:
\begin{equation*}
\star [s]_2 = -\e_{123} [s]_2 \in \bigwedge_1 \R^3
\end{equation*}
A convenient notation is used for normalized vectors. For a given non-zero vector $v \ne 0$, the \textbf{normalization function} is defined as:
\begin{equation*}
\N(v) = \frac{v}{\|v\|}
\end{equation*}
where $\|v\| = \sqrt{v \cdot v} = \sqrt{v v}$ is the length of vector $v$. Note that it is impossible to define a normalized zero vector. Let $s$ be a spinor and $v$ be a vector in $\Cl_3$. The rotation $s v s^{-1}$ of the vector $v$ by the spinor $s$ gives a rotated vector $v' = x' \e_0 + y' \e_1 + z' \e_2$ which is equal to:
\begin{align} \label{eqn:spinor_rotation_formula}
\left\{
\begin{aligned}
& x' = (a^2 - b^2 + c^2 - d^2) x + 2 y (a b + c d) + 2 z (b c - a d) \\
& y' = (a^2 - b^2 - c^2 + d^2) y + 2 x (c d - a b) + 2 z (b d + a c) \\
& z' = (a^2 + b^2 - c^2 - d^2) z + 2 x (b c + a d) + 2 y (b d - a c)
\end{aligned}
\right.
\end{align}

This paper if organized as follows: first we obtain the complete set of spinors which do not rotate a given non-zero vector. Next we derive a spinor which rotates a given initial vector to cover a given terminal vector. By combining these two formulas we derive a formula which represents an arbitrary rotation by initial and terminal vectors (swing factor) together with axis rotation (twist factor). Finally swing-twist representation is inversed for an arbitrary rotation and a unique swing-twist decomposition of a spinor is obtained. In conclusion some applications of the proposed decomposition are presented.

\section{The set of spinors which do not rotate a given vector}

In this section we assume that $v$ is a non-zero vector. The set of all spinors $s \in \Spin(3)$ which do not rotate the given vector $v$ will be called an invariant set of spinors for a given vector. The following theorem will be proved:
\begin{proposition}[The set of spinors which do not rotate a given vector] \label{thm:proposition_twist_formula}
Let $v$ be a non-zero vector in $\Cl_3$. The complete set of spinors $s$ such that $s v s^{-1} = v$ is determined by:
\begin{equation*}
s = \exp(\e_{123} \alpha \N(v))
\end{equation*}
for all $\alpha \in [0; 2 \pi)$. For a given element $a \in \Cl_3$ exponential of $a$ is defined as:
\begin{equation*}
\exp(a) := \sum_{k=0}^\infty \frac{a^k}{k!}
\end{equation*}
\end{proposition}

\begin{proof}
The proof will use the coordinate expansion of spinor rotation formula and the normalization identity giving the following set of equations:
\begin{align} \label{eqn:spinor_automorphism_set}
\left\{
\begin{aligned}
& (a^2 - b^2 + c^2 - d^2 -1) x + 2 y (a b + c d) + 2 z (b c - a d) = 0 \\
& (a^2 - b^2 - c^2 + d^2 -1) y + 2 x (c d - a b) + 2 z (b d + a c) = 0 \\
& (a^2 + b^2 - c^2 - d^2 -1) z + 2 x (b c + a d) + 2 y (b d - a c) = 0 \\
& a^2 + b^2 + c^2 + d^2 - 1 = 0
\end{aligned}
\right.
\end{align}
The key idea is to extract simple relations from \ref{eqn:spinor_automorphism_set} yet avoiding high order equations. Denote $Q_i$ as the left side of $i$th equation of \ref{eqn:spinor_automorphism_set}. A valid solution $(a, b, c, d)$ must satisfy the following equation:
\begin{equation} \label{eqn:spinor_automorphism_set_resolving_equation}
x Q_1 + y Q_2 + z Q_3 - (x^2 + y^2 + z^2) Q_4 = 0
\end{equation}
formula \eqref{eqn:spinor_automorphism_set_resolving_equation} can be expanded:
\begin{align*}
& (a^2 - b^2 + c^2 - d^2 - 1) x^2 + 2 x y (a b + c d) + 2 x z (b c - a d) \\
& + (a^2 - b^2 - c^2 + d^2 - 1) y^2 + 2 x y (c d - a b) + 2 y z (b d + a c) \\
& + (a^2 + b^2 - c^2 - d^2 - 1) z^2 + 2 x z (b c + a d) + 2 y z (b d - a c) \\
& - (x^2 + y^2 + z^2) (a^2 + b^2 + c^2 + d^2 - 1) = 0 \\
& (a^2 - b^2 + c^2 - d^2 - 1 - a^2 - b^2 - c^2 - d^2 + 1) x^2 + 2 x y c d + 2 x z b c \\
& + (a^2 - b^2 - c^2 + d^2 - 1 - a^2 - b^2 - c^2 - d^2 + 1) y^2 + 2 x y c d + 2 y z b d \\
& + (a^2 + b^2 - c^2 - d^2 - 1 - a^2 - b^2 - c^2 - d^2 + 1) z^2 + 2 x z b c + 2 y z b d = 0 \\
& (b^2 + d^2) x^2 + (b^2 + c^2) y^2 + (c^2 + d^2) z^2 - 2 x y c d - 2 x z b c - 2 y z b d = 0 \\
& (b x - c z)^2 + (d x - c y)^2 + (b y - d z)^2 = 0
\end{align*}
which implies that the three identities must hold:
\begin{equation} \label{eqn:spinor_automorphism_set_identities}
b x = c z, \quad d x = c y, \quad b y = d z
\end{equation}
Next, identities \eqref{eqn:spinor_automorphism_set_identities} are plugged into the first equation of \ref{eqn:spinor_automorphism_set} and simplified using the identity $a^2 + b^2 + c^2 + d^2 = 1$:
\begin{align*}
& 2 (a^2 + c^2 - 1) x + 2 y a b + 2 y c d + 2 z b c - 2 z a d = 0 \\
& 2 (a^2 + c^2 - 1) x + 2 z a d + 2 y c d + 2 z b c - 2 z a d = 0 \\
& (a^2 + c^2 - 1) x + y c d + z b c = 0
\end{align*}
When $x \ne 0$ there is
\begin{align}
& (a^2 + c^2 - 1) x + y c d + z b c = 0 \notag \\
& (a^2 + c^2 - 1) x + y c \frac{c y}{x} + z \frac{c z}{x} c = 0 \notag \\
& (a^2 + c^2 - 1) x^2 + y^2 c^2 + z^2 c^2 = 0 \notag \\
& c^2 (x^2 + y^2 + z^2) = x^2 (1 - a^2) \notag \\
& c = \sigma x \frac{\sqrt{1 - a^2}}{\sqrt{x^2 + y^2 + z^2}} \label{eqn:spinor_automorphism_set_general_c}
\end{align}
for $\sigma \in \{-1, 1\}$. In the other case, when $x = 0$ is is easy to observe from \eqref{eqn:spinor_automorphism_set_identities} that
\begin{align*}
c y = 0, \quad c z = 0
\end{align*}
plugging the above to \eqref{eqn:spinor_automorphism_set}:
\begin{align} \label{eqn:spinor_automorphism_set_extract_for_z_0}
\left\{
\begin{aligned}
& y a b - z a d = 0 \\
& (a^2 - b^2 - c^2 + d^2 -1) y + 2 z b d = 0 \\
& (a^2 + b^2 - c^2 - d^2 -1) z + 2 y b d = 0 \\
& a^2 + b^2 + c^2 + d^2 - 1 = 0
\end{aligned}
\right.
\end{align}
replace in the second and the third equation of \eqref{eqn:spinor_automorphism_set_extract_for_z_0} with $b y = d z$:
\begin{align} \label{eqn:spinor_automorphism_set_extract_for_z_0_simplified}
& \left\{
\begin{aligned}
& (a^2 - b^2 - c^2 + d^2 -1) y + 2 b^2 y = 0 \\
& (a^2 + b^2 - c^2 - d^2 -1) z + 2 d^2 z = 0
\end{aligned}
\right. \\
& \left\{
\begin{aligned}
& (a^2 + b^2 - c^2 + d^2 -1) y = 0 \\
& (a^2 + b^2 - c^2 + d^2 -1) z = 0
\end{aligned}
\right.
\end{align}
it is impossible that both $y = 0$ and $z = 0$ since in this case $x = 0$ and $\|v\| \ne 0$ by assumption. Hence, from any of the above equations there must be:
\begin{equation*}
a^2 + b^2 - c^2 + d^2 -1 = 0
\end{equation*}
subtracting from both sides of this equation the fourth equation of \eqref{eqn:spinor_automorphism_set_extract_for_z_0} one gets immediately:
\begin{align} \label{eqn:spinor_automorphism_set_when_x_0_then_c_0}
& (a^2 + b^2 - c^2 + d^2 -1) - (a^2 + b^2 + c^2 + d^2 - 1) = -2 c^2 = 0 \\
& \Longrightarrow c = 0
\end{align}
which is also covered by the general solution \eqref{eqn:spinor_automorphism_set_general_c} thus it can be assumed that \eqref{eqn:spinor_automorphism_set_general_c} is the only solution. Remaining spinor components are derived as follows. When $x \ne 0$ one writes using identities \eqref{eqn:spinor_automorphism_set_identities}:
\begin{align}
b & = \frac{z}{x} \sigma x \frac{\sqrt{1 - a^2}}{\sqrt{x^2 + y^2 + z^2}} = \sigma z \frac{\sqrt{1 - a^2}}{\sqrt{x^2 + y^2 + z^2}} \label{eqn:spinor_automorphism_set_general_b} \\
d & = \frac{y}{x} \sigma x \frac{\sqrt{1 - a^2}}{\sqrt{x^2 + y^2 + z^2}} = \sigma y \frac{\sqrt{1 - a^2}}{\sqrt{x^2 + y^2 + z^2}} \label{eqn:spinor_automorphism_set_general_d}
\end{align}
in the case when $x = 0$ from \eqref{eqn:spinor_automorphism_set_when_x_0_then_c_0} there is $c = 0$. To calculate $b$ and $d$ one rewrites \eqref{eqn:spinor_automorphism_set_extract_for_z_0}:
\begin{align} \label{eqn:spinor_automorphism_set_x_0}
\left\{
\begin{aligned}
& y a b - z a d = 0 \\
& (a^2 - b^2 + d^2 -1) y + 2 z b d = 0 \\
& (a^2 + b^2 - d^2 -1) z + 2 y b d = 0 \\
& a^2 + b^2 + d^2 - 1 = 0
\end{aligned}
\right.
\end{align}
Consider the case $y \ne 0$ then from \eqref{eqn:spinor_automorphism_set_identities} there is $b = \frac{z}{y} d$. Plugging this identity into the fourth equation of \eqref{eqn:spinor_automorphism_set_x_0} one obtains:
\begin{align*}
& a^2 - 1 + (\frac{z^2}{y^2} + 1) d^2 = 0 \\
& (a^2 - 1) y^2 + (x^2 + y^2 + z^2) d^2 = 0 \\
& d = \sigma y \frac{\sqrt{1 - a^2}}{\sqrt{x^2 + y^2 + z^2}}
\end{align*}
and from $b = \frac{z}{y} d$ there is:
\begin{equation*}
b = \frac{z}{y} \sigma y \frac{\sqrt{1 - a^2}}{\sqrt{x^2 + y^2 + z^2}} = \sigma z \frac{\sqrt{1 - a^2}}{\sqrt{x^2 + y^2 + z^2}}
\end{equation*}
In the case when $y = 0$ there must be $z \ne 0$ since $\|v\| \ne 0$. Thus, from \eqref{eqn:spinor_automorphism_set_identities} there is $d = \frac{y}{z} b$. Plugging this identity into the fourth equation of \eqref{eqn:spinor_automorphism_set_x_0} one obtains:
\begin{align*}
& a^2 - 1 + (\frac{y^2}{z^2} + 1) b^2 = 0 \\
& (a^2 - 1) z^2 + (x^2 + y^2 + z^2) b^2 = 0 \\
& b = \sigma z \frac{\sqrt{1 - a^2}}{\sqrt{x^2 + y^2 + z^2}}
\end{align*}
and from $b = \frac{z}{y} d$ there is:
\begin{equation*}
d = \frac{y}{z} \sigma z \frac{\sqrt{1 - a^2}}{\sqrt{x^2 + y^2 + z^2}} = \sigma y \frac{\sqrt{1 - a^2}}{\sqrt{x^2 + y^2 + z^2}}
\end{equation*}
In all cases a general solution to \eqref{eqn:spinor_automorphism_set} in coordinates is:
\begin{align} \label{eqn:spinor_rotation_formula_solution_in_coordinates}
& a \in [-1; 1], \quad \sigma \in \{ -1, 1 \} \\
& b = \sigma z \frac{\sqrt{1 - a^2}}{\sqrt{x^2 + y^2 + z^2}}, \quad c = \sigma x \frac{\sqrt{1 - a^2}}{\sqrt{x^2 + y^2 + z^2}}, \quad d = \sigma y \frac{\sqrt{1 - a^2}}{\sqrt{x^2 + y^2 + z^2}}
\end{align}
There are two parametrized solutions to the set of equations \eqref{eqn:spinor_automorphism_set}:
\begin{align*}
s = & a + \sigma z \frac{\sqrt{1 - a^2}}{\sqrt{x^2 + y^2 + z^2}} e_{12} + \\
    & \sigma x \frac{\sqrt{1 - a^2}}{\sqrt{x^2 + y^2 + z^2}} e_{23} + \sigma y \frac{\sqrt{1 - a^2}}{\sqrt{x^2 + y^2 + z^2}} e_{31} \\
s = & a + \sigma \frac{\sqrt{1 - a^2}}{\sqrt{x^2 + y^2 + z^2}} (z e_{12} + x e_{23} + y e_{31}) \\
s = & a + \sigma \frac{\sqrt{1 - a^2}}{\sqrt{x^2 + y^2 + z^2}} \e_{123} (z e_3 + x e_1 + y e_2) \\
s = & a + \sigma \frac{\sqrt{1 - a^2}}{\|v\|} \e_{123} v \\
s = & a + \sigma \e_{123} \sqrt{1 - a^2} \N(v)
\end{align*}
Finally, since many of the operations were reductions, all solutions are plugged into the original set of equations to check their validity. It can be seen that:
\begin{align*}
s v s^{-1} & = (a + \sigma \e_{123} \sqrt{1-a^2} \N(v)) v (a + \sigma \e_{123} \sqrt{1-a^2} \N(v))^{-1} \\
& = (a + \sigma \e_{123} \sqrt{1-a^2} \N(v)) v (a - \sigma \e_{123} \sqrt{1-a^2} \N(v)) \\
& = (a v + \sigma \e_{123} \sqrt{1-a^2} \N(v) v) (a - \sigma \e_{123} \sqrt{1-a^2} \N(v)) \\
& = (a v + \sigma \e_{123} \sqrt{1-a^2} \|v\|) (a - \sigma \e_{123} \sqrt{1-a^2} \N(v)) \\
& = a v a + a v (- \sigma \e_{123} \sqrt{1-a^2} \N(v)) + (\sigma \e_{123} \sqrt{1-a^2} \|v\|) a + \\
& \quad (\sigma \e_{123} \sqrt{1-a^2} \|v\|) (-\sigma \e_{123} \sqrt{1-a^2} \N(v)) \\
& = a v a + a v (-\sigma \e_{123} \sqrt{1-a^2} \N(v)) + a v (\sigma \e_{123} \sqrt{1-a^2} \N(v)) + \\
& \quad \sigma (- \sigma) \e_{123} \e_{123} \sqrt{1-a^2} \sqrt{1-a^2} \|v\| \N(v) \\
& = a v a + (1 - a^2) v = a^2 v + v - a^2 v = v
\end{align*}
which confirms that all the solutions are valid. Now we obtain the following formula:
\begin{equation*}
s = a + \sigma \e_{123} \sqrt{1-a^2} \N(v) 
\end{equation*}
for $a \in [-1; 1]$. Now it is further simplified by using trigonometric and exponential series converging for all arguments. We substitute $a = \cos(\alpha)$ to obtain:
\begin{equation*}
a + \sigma \e_{123} \sqrt{1-a^2} \N(v) = \cos(\alpha) + \sin(\alpha) \e_{123} \N(v)
\end{equation*}
where $\alpha \in [0; 2 \pi)$. Then, we express the trigonometry using exponential function. We use series expansions which are valid for all $\alpha$. Note that $(-1)^k = (\e_{123} \e_{123})^k = (\e_{123})^{2k}$ and $(\N(v))^{2 k} = 1$.
\begin{align*}
& \cos(\alpha) + \sin(\alpha) \e_{123} \N(v) = \sum_{k=0}^\infty \frac{(-1)^k}{(2 k)!} \alpha^{2 k} + \e_{123} \N(v) \sum_{k=0}^\infty \frac{(-1)^k}{(2 k + 1)!} \alpha^{2 k + 1} \\
& = \sum_{k=0}^\infty \frac{(\e_{123})^{2k}}{(2 k)!} \alpha^{2 k} (\N(v))^{2 k} + \e_{123} \N(v) \sum_{k=0}^\infty \frac{(\e_{123})^{2k}}{(2 k + 1)!} \alpha^{2 k + 1} (\N(v))^{2 k} \\
& = \sum_{k=0}^\infty \frac{(\e_{123})^{2k}}{(2 k)!} \alpha^{2 k} (\N(v))^{2 k} + \sum_{k=0}^\infty \frac{(\e_{123})^{2k + 1}}{(2 k + 1)!} \alpha^{2 k + 1} (\N(v))^{2 k + 1} \\
& = \sum_{k=0}^\infty \frac{(\e_{123} \alpha \N(v))^{2k}}{(2 k)!} + \sum_{k=0}^\infty \frac{(\e_{123} \alpha \N(v))^{2k + 1}}{(2 k + 1)!} \\
& = \sum_{k=0}^\infty \frac{(\e_{123} \alpha \N(v))^{k}}{k!} = \exp(\e_{123} \alpha \N(v))
\end{align*}
which is precisely the stated general formula. This completes the proof.
\end{proof}

\section{A spinor which rotates a given vector}

The following lemma states what is a spinor which rotates a given vector to be equal to a target one. In literature it is sometimes called \textit{direct rotation}, as in \cite{baerlocher_boulic}. Schematic view is presented in figure \ref{img:direct_rotation}.
\begin{figure}[ht]
\centering
\includegraphics[width=3cm]{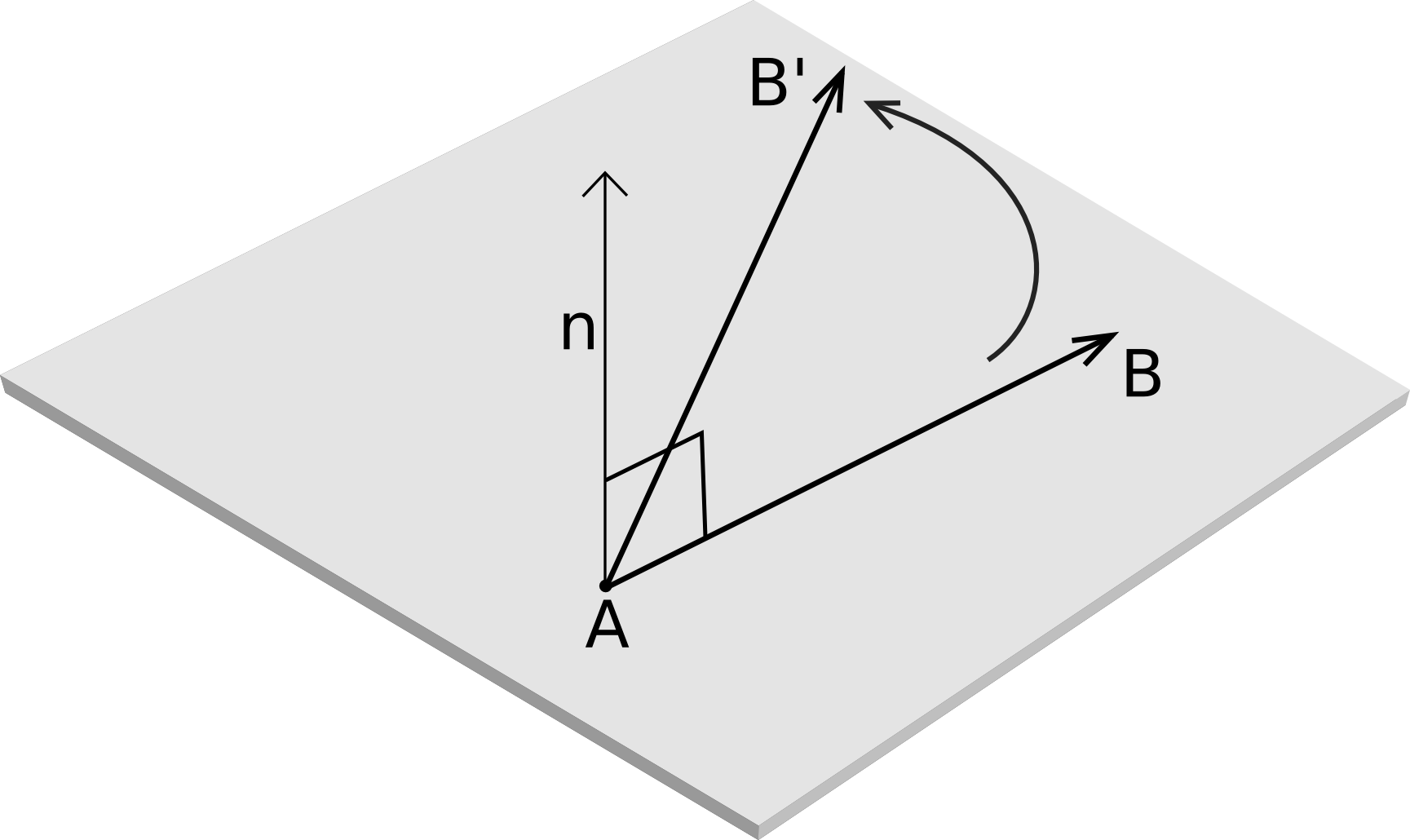}
\caption{A direct rotation of $AB$ to $AB'$}
\label{img:direct_rotation}
\end{figure}
The lemma is required to prove a general theorem about the set of spinors which rotate a given initial vector to a given target vector. We start with the following initial lemma which applies to the Clifford algebra of $\R^2$. At first the following basic property of $\Cl_2$ is proved:
\begin{lemma} \label{thm:lemma_external_product}
The external product of two given vectors $v$ and $w$ in $\Cl_2$ can be written as:
\begin{equation*}
w \wedge v = \frac{1}{2} \sgn(v \wedge w \e_{12}) \|v - w\| \|v + w\|\e_{12}
\end{equation*}
\end{lemma}
\begin{proof}
Assume that $v = x \e_1 + y \e_2$ and $w = p \e_1 + q \e_2$. In special case when $w = \pm v$ both sides of the equation are equal to zero therefore the equation holds. In general case $w \ne \pm v$ and one can rewrite both sides in coordinates. Left-hand side is equal to:
\begin{align*}
w \wedge v & = \frac{1}{2} ( w v - v w ) = \frac{1}{2} \left( (p \e_1 + q \e_2)(x \e_1 + y \e_2) - (x \e_1 + y \e_2)(p \e_1 + q \e_2) \right) \\
& = \frac{1}{2} \left( 2 p y \e_{12} - 2 q x \e_{12} \right) = (p y - q x) \e_{12}
\end{align*}
Right-hand side is equal to:
\begin{align*}
& \frac{1}{2} \sgn(v \wedge w \e_{12}) \|v - w\| \|v + w\|\e_{12} \\
& = \frac{1}{2} \sgn(\frac{1}{2} ( v w - w v ) \e_{12}) \sqrt{(v - w)^2} \sqrt{(v + w)^2} \e_{12} \\
& = \frac{1}{2} \sgn(( (x \e_1 + y \e_2)(p \e_1 + q \e_2) \\
& \quad - (p \e_1 + q \e_2)(x \e_1 + y \e_2) ) \e_{12}) \sqrt{(v - w)^2 (v + w)^2 } \e_{12} \\
& = \frac{1}{2} \sgn(( 2 q x \e_{12} - 2 p y \e_{12} ) \e_{12}) \\
& \quad \sqrt{ ((x - p)^2 + (y - q)^2)((x+p)^2+(y+q)^2)} \e_{12} \\
& = \frac{1}{2} \sgn(p y - q x) \sqrt{ 4 (p y - q x)^2 } \e_{12} = \sgn(p y - q x) |p y - q x| \e_{12} \\
& = (p y - q x) \e_{12}
\end{align*}
Both sides are equal so it completes the proof.
\end{proof}

\begin{lemma}
Let $v$ and $w$ be a pair of non-zero vectors of equal length in $\Cl_2$ such that $v \ne -w$. Then there exist a pair of spinors $\pm s \in \Spin(2)$ which rotate $v$ so that the rotated vector is equal to $w$:
\begin{equation*}
s v s^{-1} = w
\end{equation*}
then the pair of spinors is equal to:
\begin{equation*}
s = \pm \N(v + w) \N(v)
\end{equation*}
\end{lemma}
\begin{proof}
Let $s = a + b \e_{12} \in \Spin(2)$ be a spinor and $v = x \e_1 + y \e_2$, $w = p \e_1 + q \e_2$ be a pair of given vectors. The equation can be expressed in coordinates with the following set of equations:
\begin{align} \label{eqn:rotating_spinor_lemma_initial_set}
\left\{
\begin{aligned}
& (a^2 - b^2) x + 2 a b y = p \\
& (a^2 - b^2) y - 2 a b x = q
\end{aligned}
\right.
\end{align}
From the assumption it is impossible that both $x$ and $y$ are simultaneously zero. When $x \ne 0$ we solve the set of equations:
\begin{align*}
\left\{
\begin{aligned}
& a^2 - b^2 = \frac{p - 2 a b y}{x}\\
& \frac{p - 2 a b y}{x} y - 2 a b x = q
\end{aligned}
\right.
\end{align*}
The second equation is then simplified:
\begin{align*}
& (p - 2 a b y) y - 2 a b x^2 - q x = 0 \\
& 2 a b = \frac{p y - q x}{x^2 + y^2}
\end{align*}
Which is then plugged into the first equation of \eqref{eqn:rotating_spinor_lemma_initial_set}:
\begin{align*}
& (a^2 - b^2) x + \frac{p y - q x}{x^2 + y^2} y = p \\
& (2 a^2 - 1) x = p - \frac{p y^2 - q x y}{x^2 + y^2} \\
& 2 a^2 = \frac{p x^2 + p y^2 - p y^2 + q x y}{x (x^2 + y^2)} + 1 = \frac{p x + q y}{x^2 + y^2} + 1 \\
& a = \pm \sqrt{\frac{x^2 + y^2 + p x + q y}{x^2 + y^2} }
\end{align*}
The same result can be obtained when we consider the other case $y \ne 0$:
\begin{align*}
\left\{
\begin{aligned}
& a^2 - b^2 = \frac{q + 2 a b x}{y}\\
& \frac{q + 2 a b x}{y} x + 2 a b y = p
\end{aligned}
\right.
\end{align*}
The second equation can be simplified:
\begin{align*}
& (q + 2 a b x) x + 2 a b y^2 - p y = 0 \\
& 2 a b = \frac{p y - q x}{x^2 + y^2}
\end{align*}
Which is then plugged into the second equation of \eqref{eqn:rotating_spinor_lemma_initial_set}:
\begin{align*}
& (a^2 - b^2) y - \frac{p y - q x}{x^2 + y^2} x = q \\
& (2 a^2 - 1) y = q + \frac{p x y - q x^2}{x^2 + y^2} \\
& 2 a^2 = \frac{q x^2 + q y^2 + p x y - q x^2}{y (x^2 + y^2)} + 1 = \frac{p x + q y}{x^2 + y^2} + 1 \\
& a = \pm \sqrt{\frac{x^2 + y^2 + p x + q y}{2 (x^2 + y^2)} }
\end{align*}

We obtained one general solution which is correct for all $x, y$. Now, for each $\sigma \in \{ -1, 1 \}$ of $a = \sigma \sqrt{\frac{x^2 + y^2 + p x + q y}{2 (x^2 + y^2)} }$ we have exactly one corresponding solution $b$. It can be calculated by plugging it into $2 a b = \frac{p y - q x}{2 (x^2 + y^2)}$ valid for all $x, y$. When $a \ne 0$ then:
\begin{align*}
& b = \frac{p y - q x}{2 a (x^2 + y^2)} \\
& b = \frac{(p y - q x) \sqrt{2 (x^2 + y^2)}}{2 \sigma \sqrt{x^2 + y^2 + p x + q y} (x^2 + y^2)} \\
& b = \sigma \frac{p y - q x}{\sqrt{x^2 + y^2 + p x + q y} \sqrt{2 (x^2 + y^2)}} \\
& b = \sigma \frac{(p y - q x)\sqrt{x^2 + y^2 - p x - q y}}{\sqrt{(x^2 + y^2)^2 - (p x + q y)^2} \sqrt{2 (x^2 + y^2)}} \\
& b = \sigma \frac{(p y - q x)\sqrt{x^2 + y^2 - p x - q y}}{\sqrt{(x^2 + y^2)^2 - (p x + q y)^2} \sqrt{2 (x^2 + y^2)}} \\
& b = \sigma \frac{(p y - q x)\sqrt{x^2 + y^2 - p x - q y}}{\sqrt{x^4 + y^4 + 2 x^2 y^2 - p^2 x^2 - q^2 y^2 - 2 p q x y} \sqrt{2 (x^2 + y^2)}} \\
& b = \sigma \frac{(p y - q x)\sqrt{x^2 + y^2 - p x - q y}}{\sqrt{x^4 + y^4 + 2 x^2 y^2 - (x^2 + y^2 - q^2) x^2 - (x^2 + y^2 - p^2) y^2 - 2 p q x y}} \cdot \\
& \frac{1}{\sqrt{2 (x^2 + y^2)}} \\
& b = \sigma \frac{p y - q x}{\sqrt{(p y - q x)^2}} \sqrt{\frac{x^2 + y^2 - p x - q y}{2 (x^2 + y^2)}}
\end{align*}
note that $(p y - q x) \e_{12} = w \wedge v$ but it is assumed that $v$ is not parallel to $w$ therefore $p y - q x \ne 0$. Since $\e_{12}^2 = -1$ the following applies:
\begin{equation*}
\frac{p y - q x}{\sqrt{(p y - q x)^2}} = \sgn(v \wedge w \e_{12})
\end{equation*}
and $b$ is equal to:
\begin{equation*}
b = \sigma \sgn(v \wedge w \e_{12})\sqrt{\frac{x^2 + y^2 - p x - q y}{2 (x^2 + y^2)}}
\end{equation*}

When $a = 0$ then from $b^2 = 1 - a^2$ we know that $b = \pm 1$ which is also a special case of the general solution.
Finally, we simplify the formula to a coordinate-free solution. For each $\sigma \in \{ -1, 1 \}$ there is a solution:
\begin{align*}
& s = \sigma \sqrt{\frac{x^2 + y^2 + p x + q y}{2 (x^2 + y^2)} } + \sigma \sgn(v \wedge w \e_{12})\sqrt{\frac{x^2 + y^2 - p x - q y}{2 (x^2 + y^2)}} \e_{12} \notag \\
& s = \frac{\sigma}{\sqrt{2 (x^2 + y^2)}} ( \sqrt{x^2 + y^2 + p x + q y} + \\
& \sgn(v \wedge w \e_{12}) \sqrt{x^2 + y^2 - p x - q y} \e_{12} ) \\
& s = \frac{\sigma}{\sqrt{2 (x^2 + y^2)}} ( \sqrt{\frac{1}{2} (x^2 + y^2 + p^2 + q^2 + 2 (p x + q y) )} + \\
& \sgn(v \wedge w \e_{12}) \sqrt{\frac{1}{2} (x^2 + y^2 + p^2 + q^2 - 2 (p x + q y) ) } \e_{12} ) \\
& s = \frac{\sigma}{2 \sqrt{x^2 + y^2}} ( \sqrt{x^2 + y^2 + p^2 + q^2 + 2 (p x + q y)} + \\
& \sgn(v \wedge w \e_{12}) \sqrt{x^2 + y^2 + p^2 + q^2 - 2 (p x + q y) } \e_{12} ) \\
& s = \frac{\sigma}{2 \|v\|} ( \sqrt{v v + w w + v w + w v} + \\
& \sgn(v \wedge w \e_{12}) \sqrt{v v + w w - v w - w v } \e_{12} ) \\
& s = \frac{\sigma}{2 \|v\|} ( \sqrt{(v + w)^2} + \sgn(v \wedge w \e_{12}) \sqrt{(v - w)^2} \e_{12} ) \\
& s = \frac{\sigma}{2 \|v\|} ( \|v + w\| + \sgn(v \wedge w \e_{12}) \|v - w\| \e_{12} ) \\
\end{align*}
From the assumptions, there is $\|v + w\| \ne 0$ so:
\begin{equation*}
s = \frac{\sigma}{2 \|v\| \|v + w\|} ( \|v + w\|^2 + \sgn(v \wedge w \e_{12}) \|v - w\| \|v + w\|\e_{12} )
\end{equation*}
Next, the basic identity from lemma \ref{thm:lemma_external_product} is used so the formula can be rewritten with only external product:
\begin{equation*}
w \wedge v = \frac{1}{2} \sgn(v \wedge w \e_{12}) \|v - w\| \|v + w\|\e_{12}
\end{equation*}
It reads that the external product is equal to a bivector whose area is equal to the half of the area of the parallelogram spanned by the both vectors and with the sign adjusted to the sign of the external product. Using the above identity, the following further simplifications are possible:
\begin{align*}
& s = \frac{\sigma}{2 \|v\| \|v + w\|} ( \|v + w\|^2 + 2 v \wedge w ) \\
& s = \frac{\sigma}{2 \|v\| \|v + w\|} ( v v + w w + 2 v \cdot w + 2 v \wedge w ) \\
& s = \frac{\sigma}{\|v\| \|v + w\|} ( v v + w v ) \\
& s = \frac{\sigma}{\|v\| \|v + w\|} (v + w ) v \\
& s = \pm \N(v + w) \N(v)
\end{align*}
Which is precisely the stated formula.
\end{proof}

We observe that the same vector formula applies to any Clifford algebra of $n$-dimensional real space. Thus, we propose the following:

\begin{proposition} \label{thm:proposition_swing_formula}
Let $v$ and $w$ be a pair of non-zero vectors of equal length in $\Cl_n$ such that $v \ne -w$. Then there exist a pair of spinors each of which rotate $v$ so that it is equal to $w$:
\begin{equation*}
s v s^{-1} = w
\end{equation*}
the pair of spinors is equal to:
\begin{equation*}
s = \pm \N(v + w) \N(v)
\end{equation*}
\end{proposition}
\begin{proof}
Instead of the original equation $s v s^{-1} = w$ an equivalent equation will be proved:
\begin{equation*}
s v = w s
\end{equation*}
Let $v$ and $w$ are the given vectors. Start with:
\begin{equation*}
v - v = v - v
\end{equation*}
since $vv = ww > 0$ each element can be multiplied by the scalar $vv$ or $ww$:
\begin{align*}
& v v v - w w v = w v v - w v v \\
& v v v + w v v = w v v + w w v \\
& (v + w) v v = w (v + w) v
\end{align*}
it is assumed that $v \ne -w$ so $\|v + w\| > 0$ and both sides can be divided by $\|v + w\| \|v\|$:
\begin{equation*}
\N(v + w) \N(v) v = w \N(v + w) \N(v)
\end{equation*}

To ensure that $s$ is a spinor it is sufficient to check whether it is a direct product of a scalar and a bivector and whether its norm is equal to 1. Indeed:
\begin{align*}
s = \N(v + w) \N(v) = \frac{vv + wv}{\|v + w\|\|v\|} = \frac{\|v\|^2 + wv}{\|v + w\|\|v\|} \in \bigwedge_0 \R^3 \otimes \bigwedge_2 \R^3
\end{align*}
and since
\begin{align*}
& s^{-1} = \N(\|v\|^2 + w \cdot v + w \wedge v)^{-1} \\
& = \N(\|v\|^2 + v \cdot w + v \wedge w) = \N(vv + vw) = \N(v) \N(v + w)
\end{align*}
the norm is:
\begin{equation*}
s s^{-1} = \N(v + w) \N(v) \N(v) \N(v + w) = \| \N(v+w) \|^2 \| \N(v) \|^2 = 1
\end{equation*}

Which completes the proof.
\end{proof}

Since the theorem is valid for $\Cl_n$ it is also valid for $\Cl_3$. In more general conclusion, for any dimension there exists a simple formula giving a pair of spinors rotating a given vector so that it covers another given vector of the same length.

\section{The set of spinors which rotate a given initial vector to a given target vector}
\label{sec:two_st_decompositions}

There are two ways of defining a swing-twist representation (composition or decomposition):
\begin{itemize}
\item twist $q$ is done \textbf{before} swing $p$; for a given spinor $r$ we have: $s = p q$. \\
This is the \textbf{swing-after-twist} representation.
\item twist $q$ is done \textbf{after} swing $p$; for a given spinor $r$ we have: $s = q p$. \\
This is the \textbf{twist-after-swing} representation.
\end{itemize}

The swing-after-twist representation is used in \cite{huyghe} while the twist-after-swing representation is usually used in the direct method. Since in practice both conventions are used (the first one is slightly less common) in this paper both decompositions will be presented and proved. The following theorem is a general solution to Clifford product equation:
\begin{equation*}
s v s^{-1} = w
\end{equation*}
We prove the following
\begin{proposition} \label{thm:proposition_st_representation}
Let $v$ and $w$ be a pair of non-zero vectors of equal length in $\Cl_3$ such that $v \ne -w$. The set of spinor solutions $s \in \Spin(3)$ to
\begin{equation*}
s v s^{-1} = w
\end{equation*}
in the case of swing-after-twist representation is equal to:
\begin{equation*}
s = \pm \N(v + w) \N(v) \exp(\e_{123} \alpha \N(v))
\end{equation*}
and in the case of twist-after-swing representation is equal to:
\begin{equation*}
s = \pm \exp(\e_{123} \alpha \N(w)) \N(v + w) \N(v)
\end{equation*}
\end{proposition}
\begin{proof}
The solution is the set of all possible rotations which swing axis from initial $v$ to terminal $w$ with any possible twist during the movement. Using propositions \ref{thm:proposition_twist_formula} and \ref{thm:proposition_swing_formula} one composes swing and twist according to the order used in a given representation. In the case of swing-after-twist representation twist $q = \exp(\e_{123} \alpha \N(v))$ factor (around $v$ axis) is applied at first and only after it, swing factor $p = \N(v + w) \N(v)$ is applied (swinging the axis from $v$ to $w$). In the other case of twist-after-swing representation swing $p = \N(v + w) \N(v)$ is applied first (swinging the axis from $v$ to $w$) and after that twist factor $q = \exp(\e_{123} \alpha \N(w))$ is applied with respect to the terminal axis $w$.
\end{proof}

\section{Swing-twist decomposition of a spinor}

This is the main result of this paper. In this section the inverse of formulas given in proposition \ref{thm:proposition_st_representation} is derived. Given a spinor it is possible to calculate its decomposition into twist and swing factors in respect to a given non-zero vector. Since there are two different swing-twist representations, there are also two swing-twist decompositions for swing-after-twist and twist-after-swing representation respectively. In this section, the initial vector will be called a base vector.

\begin{theorem}[Swing-twist decomposition of a spinor in swing-after-twist representation] \label{thm:swing_twist_decomposition_sat_type}
Assume that $s \in \Spin(3)$ is a spinor. For any non-zero base vector $v \in \Cl_3$ such that $s v s^{-1} \ne -v$ there exist \textbf{a unique up to the sign} swing-twist decomposition in swing-after-twist representation
\begin{equation*}
s = \pm p q
\end{equation*}
where swing spinor $p$ and twist spinor $q$ are equal to:
\begin{align*}
& p = \pm s \tilde{\sigma}_v(s) \\
& q = \pm \sigma_v(s)
\end{align*}
where $\sigma_v(s): \Spin(3) \longrightarrow \Spin(3)$ is a function of spinor $s$:
\begin{equation*}
\sigma_b(s) = \N(v (v \cdot s))
\end{equation*}
the reversion $\tilde{\sigma}_v(s)$ is equal to:
\begin{equation*}
\tilde{\sigma}_v(s) = \N(v (v \cdot \tilde{s}))
\end{equation*}
The function $\sigma_v(s)$ will be called \textbf{a twist projection function}. 
\end{theorem}
\begin{proof}
The proof is divided into several steps. At first the problem is reformulated in coordinates. Then, twist angle is calculated and twist spinor and finally swing spinor. Denote the following:
\begin{align*}
& s = a + b \e_{12} + c \e_{23} + d \e_{31} \\
& v = v_x \e_1 + v_y \e_2 + v_z \e_3 \\
& w = w_x \e_1 + w_y \e_2 + w_z \e_3
\end{align*}
The angle $\alpha$ will be described implicitly by:
\begin{align*}
& \cos(\alpha) = k \\
& \sin(\alpha) = l \\
& k^2 + l^2 = 1
\end{align*}
Vector $w$ exists on the assumption that $s v s^{-1} \ne -v$. What does this assumption require about the base vector and the spinor is explained in the appendix of this paper. At first, the equation is rewritten in coordinates. The swing factor is equal to:
\begin{align*}
p & = \N(w + v) \N(v) \\
& = \frac{( (w_x + v_x) \e_1 + (w_y + v_y) \e_2 + (w_z + v_z) \e_3)(v_x \e_1 + v_y \e_2 + v_z \e_3)}{\sqrt{(w_x + v_x)^2 + (w_y + v_y)^2 + (w_z + v_z)^2}\sqrt{v_x^2 + v_y^2 + v_z^2}} \\
& = ((w_x + v_x)^2 + (w_y + v_y)^2 + (w_z + v_z)^2)^{-\frac{1}{2}}(v_x^2 + v_y^2 + v_z^2)^{-\frac{1}{2}} \cdot \\
& [ (w_x + v_x) v_x + (w_y + v_y) v_y + (w_z + v_z) v_z + \\
& ((w_x + v_x) v_y - (w_y + v_y) v_x) \e_{12} + ((w_y + v_y) v_z - (w_z + v_z) v_y) \e_{23} + \\
& ((w_z + v_z) v_x - (w_x + v_x) v_z) \e_{31} ]
\end{align*}
Twist factor is equal to:
\begin{equation} \label{eqn:general_formula_for_twist}
q = \cos(\alpha) + \e_{123} \N(v) \sin(\alpha) = k + \frac{l}{\sqrt{v_x^2 + v_y^2 + v_z^2}} (b_z \e_{12} + b_x \e_{23} + b_y \e_{31})
\end{equation}
Combining the above formulas there is:
\begin{align*}
pq & = ((w_x + v_x)^2 + (w_y + v_y)^2 + (w_z + v_z)^2)^{-\frac{1}{2}}(v_x^2 + v_y^2 + v_z^2)^{-\frac{1}{2}} \cdot \\
& [ (w_x + v_x) v_x + (w_y + v_y) v_y + (w_z + v_z) v_z + \\
& ((w_x + v_x) v_y - (w_y + v_y) v_x) \e_{12} + ((w_y + v_y) v_z - (w_z + v_z) v_y) \e_{23} + \\
& ((w_z + v_z) v_x - (w_x + v_x) v_z) \e_{31} ] \cdot \\
& [k + \frac{l}{\sqrt{v_x^2 + v_y^2 + v_z^2}} (b_z \e_{12} + b_x \e_{23} + b_y \e_{31}) ] \\
& = ((w_x + v_x)^2 + (w_y + v_y)^2 + (w_z + v_z)^2)^{-\frac{1}{2}}(v_x^2 + v_y^2 + v_z^2)^{-\frac{1}{2}} \cdot \\
& [ k ( (w_x + v_x) v_x + (w_y + v_y) v_y + (w_z + v_z) v_z) + \\
& ( k (w_x v_y - w_y v_x) + l \sqrt{v_x^2 + v_y^2 + v_z^2} (w_z + v_z) ) \e_{12} + \\
& ( k (w_y v_z - w_z v_y) + l \sqrt{v_x^2 + v_y^2 + v_z^2} (w_x + v_x) ) \e_{23} + \\
& ( k (w_z v_x - w_x v_z) + l \sqrt{v_x^2 + v_y^2 + v_z^2} (w_y + v_y) ) \e_{31} ]
\end{align*}

Two spinors are equal if and only if corresponding coefficients are equal. Therefore the following set of equations determines the solution:
\begin{equation} \label{eqn:set_of_equations_decomposition}
\begin{cases}
a = [ k ( (w_x + v_x) v_x + (w_y + v_y) v_y + (w_z + v_z) v_z) ] \Delta \\
b = [ k (w_x v_y - w_y v_x) + l \sqrt{v_x^2 + v_y^2 + v_z^2} (w_z + v_z) ] \Delta \\
c = [ k (w_y v_z - w_z v_y) + l \sqrt{v_x^2 + v_y^2 + v_z^2} (w_x + v_x) ] \Delta \\
d = [ k (w_z v_x - w_x v_z) + l \sqrt{v_x^2 + v_y^2 + v_z^2} (w_y + v_y) ] \Delta \\
k^2 + l^2 = 1 \\
w_x^2 + w_y^2 + w_z^2 = v_x^2 + v_y^2 + v_z^2 \\
a^2 + b^2 + c^2 + d^2 = 1
\end{cases}
\end{equation}
Where $\Delta = ((w_x + v_x)^2 + (w_y + v_y)^2 + (w_z + v_z)^2)^{-\frac{1}{2}}(v_x^2 + v_y^2 + v_z^2)^{-\frac{1}{2}}$. By substituting variables in the set of equations \eqref{eqn:set_of_equations_decomposition} it is easy to rise unfavourably the degree of involved polynomials. Thus the following careful operations are performed. First, note that it is easy to obtain variable $k$ from the first equation:
\begin{equation} \label{eqn:deomcposition_first_equation_unsimplified}
a = \frac{k (w + v) \cdot v}{\|w + v\| \|v\|}
\end{equation}
Observe that:
\begin{equation*}
\|w + v\|^2 = (w + v)(w + v) = w w + v v + w v + v w = 2 ( v \cdot v + v \cdot w) = 2 v \cdot (w + v)
\end{equation*}
Using this, one simplifies \eqref{eqn:deomcposition_first_equation_unsimplified} to:
\begin{equation*}
a = \frac{k (w + v) \cdot v}{\|w + v\| \|v\|} = \frac{k \|w + v\|^2}{\|w + v\| \|v\|} = k \frac{\|w + v\|}{2 \|v\|}
\end{equation*}
From that one obtains the value of $k$:
\begin{equation} \label{eqn:formula_for_k}
k = a \frac{2 \|v\|}{\|w + v\|}
\end{equation}
The value of $k$ is now put into the set of equations \eqref{eqn:set_of_equations_decomposition} which then can be simplified to:
\begin{equation} \label{eqn:set_of_equations_decomposition_simplified}
\begin{cases}
b \|w + v\| = 2 a \frac{1}{\|w + v\|} (w_x v_y - w_y v_x) + l (w_z + v_z) \\
c \|w + v\| = 2 a \frac{1}{\|w + v\|} (w_y v_z - w_z v_y) + l (w_x + v_x) \\
d \|w + v\| = 2 a \frac{1}{\|w + v\|} (w_z v_x - w_x v_z) + l (w_y + v_y)
\end{cases}
\end{equation}
It is a set of three linear equations with three unknowns $M [w_x, w_y, w_z]^T = N$. The characteristic matrix $M$ is:
\begin{equation*}
M =
\begin{vmatrix}
2 a \frac{1}{\|w + v\|} v_y & -2 a \frac{1}{\|w + v\|} v_x & l \\
l & 2 a \frac{1}{\|w + v\|} v_z & -2 a \frac{1}{\|w + v\|} v_y \\
-2 a \frac{1}{\|w + v\|} v_z & l & 2 a \frac{1}{\|w + v\|} v_x
\end{vmatrix}
\end{equation*}
and the vector $N$ is equal to:
\begin{equation*}
N = [b \|w + v\| - l v_z, c \|w + v\| - l v_x, d \|w + v\| - l v_y]^T
\end{equation*}
Any known method can be used to compute the determinant of $M$ which is equal to:
\begin{equation*}
det(M) = l
\end{equation*}
The case $l = 0$ needs special care and will be now discussed separately. In this case, from \eqref{eqn:general_formula_for_twist} twist factor is equal to:
\begin{equation*}
q = k + \frac{l}{\sqrt{v_x^2 + v_y^2 + v_z^2}} (b_z \e_{12} + b_x \e_{23} + b_y \e_{31}) = k
\end{equation*}
Since $q$ is a spinor, its component $k$ must be equal to:
\begin{align*}
k = \pm 1
\end{align*}
which then implies that twist and swing factors are equal to:
\begin{align*}
p & = \pm s \\
q & = \pm 1
\end{align*}
It is a special case of the general formula. It holds for all vector $v$ and spinors $s$ such that $\alpha = z \pi, \, z \in \mathbb{Z}$. In the case of $l \ne 0$ there exist exactly one solution to \eqref{eqn:set_of_equations_decomposition_simplified}. After having the matrix inverted and solution calculated, one gets:
\begin{align*}
& w_x = [4 a^2 (v_x^2 + v_y^2 + v_z^2) l + \|w + v\|^2 l^3]^{-1} \\
& [ 2 a (b v_y - d v_z) \|w + v\|^2 l + \\
& \|w + v\|^2 l^2 (c \|w + v\| - v_x l) + 4 a^2 v_x ((b v_z + v_x c + v_y d) \|w + v\| - (v_x^2 + v_y^2 + v_z^2) l) ] \\
& w_y = [4 a^2 (v_x^2 + v_y^2 + v_z^2) l + \|w + v\|^2 l^3]^{-1} \\
& [ 2 a (-b v_x + c v_z) \|w + v\|^2 l + \\
& \|w + v\|^2 l^2 (d \|w + v\| - v_y l) + 4 a^2 v_y ((b v_z + v_x c + v_y d) \|w + v\| - (v_x^2 + v_y^2 + v_z^2) l) ] \\
& w_z = [4 a^2 (v_x^2 + v_y^2 + v_z^2) l + \|w + v\|^2 l^3]^{-1} \\
& [ 2 a (c v_y + d v_x) \|w + v\|^2 l + \\
& \|w + v\|^2 l^2 (b \|w + v\| - v_z l) + 4 a^2 v_z ((b v_z + v_x c + v_y d) \|w + v\| - (v_x^2 + v_y^2 + v_z^2) l) ]
\end{align*}
All three denominators are equal and can be simplified:
\begin{align*}
& 4 a^2 (v_x^2 + v_y^2 + v_z^2) l + \|w + v\|^2 l^3 \\
& = l (4 a^2 \|v\|^2 + \|w + v\|^2 (1 - a^2 \frac{4 \|v\|^2}{\|w+v\|^2})) = l \|w + v\|^2
\end{align*}
A compact form is achieved with the following simplifications:
\begin{align*}
\begin{vmatrix} w_x \\ w_y \\ w_z \end{vmatrix} & = \frac{1}{l \|w + v\|^2} ( 2 a \begin{vmatrix} b v_y - d v_z \\ c v_z - b v_x \\ d v_x - c v_y \end{vmatrix} \|w + v\|^2 l + \|w + v\|^2 l^2 \begin{vmatrix} c \|w + v\| - v_x l \\ d \|w + v\| - v_y l \\ b \|w + v\| - v_z l \end{vmatrix} + \\
& 4 a^2 \begin{vmatrix} v_x \\ v_y \\ v_z \end{vmatrix} ( ( b v_z + c v_x + d v_y) \|w + v\| - \|v\|^2 l ) ) \\
& = 2 a \begin{vmatrix} v_x \\ v_y \\ v_z \end{vmatrix} \times \begin{vmatrix} c \\ d \\ b \end{vmatrix} + l \|w + v\| \begin{vmatrix} c \\ d \\ b \end{vmatrix} \\
& - l^2 \begin{vmatrix} v_x \\ v_y \\ v_z \end{vmatrix} + \frac{4 a^2}{l \|w + v\|^2} \begin{vmatrix} v_x \\ v_y \\ v_z \end{vmatrix} ( \|w + v\| \begin{vmatrix} v_x \\ v_y \\ v_z \end{vmatrix} \cdot \begin{vmatrix} c \\ d \\ b \end{vmatrix} - \|v\|^2 l )
\end{align*}

Vector $[c, d, b]^T$ can be formulated with coefficients of spinor $s$:
\begin{equation*}
\star [s]_2 = -\e_{123} [s]_2 = -\e_{123} (b \e_{12} + c \e_{23} + d \e_{31} ) = c \e_1 + d \e_2 + b \e_3
\end{equation*}
With this substitution further simplification is possible:
\begin{equation*}
w = 2 a v \times \star[s]_2 + l \|w + v\| \star[s]_2 - l^2 v + \frac{4 a^2}{l \|w + v\|^2} v (\|w + v\| v \cdot \star[s]_2 - \|v\|^2 l)
\end{equation*}
but form previous equations there is:
\begin{equation} \label{eqn:formula_ll}
l^2 = 1 - a^2 \frac{4 \|v\|^2}{\|w + v\|^2}
\end{equation}
so after substitution one obtains:
\begin{equation*}
w = 2 a v \times \star[s]_2 + l \|w + v\| \star[s]_2 - v + \frac{4 a^2 v}{l \|w + v\|} v \cdot \star[s]_2
\end{equation*}

The above formula is useful to compute $\|w + v\|$. Note that one can move $-v$ to the left hand side of the equation and compute lengths of both sides.
\begin{align*}
\|w + v\|^2 & = (2 a v \times \star[s]_2 + l \|w + v\| \star[s]_2 + \frac{4 a^2 v}{l \|w + v\|^2} v \cdot \star[s]_2)^2 \\
& = (2 a v \times \star[s]_2)^2 + (l \|w + v\| \star[s]_2)^2 + (\frac{4 a^2 v}{l \|w + v\|^2} v \cdot \star[s]_2)^2 \\
& + 2 [ (2 a v \times \star[s]_2) \cdot (l \|w + v\| \star[s]_2) + (2 a v \times \star[s]_2) \cdot (\frac{4 a^2 v}{l \|w + v\|^2} v \cdot \star[s]_2) \\
& + (l \|w + v\| \star[s]_2) \cdot (\frac{4 a^2 v}{l \|w + v\|^2} v \cdot \star[s]_2) ] \\
& = 4 a^2 (b \times \star[s]_2)^2 + l^2 \|w + v\|^2 (\star[s]_2)^2 + \frac{16 a^4 \|v\|^2}{l^2 \|w+v\|^2}(v\cdot \star[s]_2)^2 + 8 a^2 (v\cdot \star[s]_2)^2
\end{align*}
From the property that $(b \times \star[s]_2)^2 + (b \cdot \star[s]_2)^2 = \|b\|^2 (\star[s]_2)^2$ the following formula is obtained:
\begin{equation*}
\|w + v\|^2 = 4 a^2 \|b\|^2 (\star[s]_2)^2 + l^2 \|w + v\|^2 (\star[s]_2)^2 + 4 a^2 (v\cdot \star[s]_2)^2 + \frac{16 a^4 \|v\|^2}{l^2 \|w+v\|^2}(v\cdot \star[s]_2)^2
\end{equation*}
There is:
\begin{equation} \label{eqn:subst_llwv}
l^2 \|w + v\|^2 = \left(1 - \frac{4 a^2 \|v\|^2}{\|w+v\|^2} \right) \|w + v\|^2 = \|w + v\|^2 - 4 a^2 \|v\|^2
\end{equation}
So one can write the following equivalent formula:
\begin{align*}
\|w + v\|^2 & = 4 a^2 \|b\|^2 (\star[s]_2)^2 + (\|w + v\|^2 - 4 a^2 \|v\|^2) (\star[s]_2)^2 + \\
& 4 a^2 (v\cdot \star[s]_2)^2 + \frac{16 a^4 \|v\|^2}{l^2 \|w+v\|^2}(v\cdot \star[s]_2)^2 \\
& = \|w + v\|^2 (\star[s]_2)^2 + 4 a^2 (v\cdot \star[s]_2)^2 + \frac{16 a^4 \|v\|^2}{l^2 \|w+v\|^2}(v\cdot \star[s]_2)^2
\end{align*}
Substituting $(\star[s]_2)^2 = 1 - a^2$ and once again \eqref{eqn:subst_llwv} one obtains:
\begin{equation*}
\|w + v\|^2 = \|w + v\|^2 (1 - a^2) + (v \cdot \star[s]_2)^2 \left( \frac{4 a^2 (\|w + v\|^2 - 4 a^2 \|v\|^2) + 16 a^4 \|v\|^2}{l^2 \|w+v\|^2} \right)
\end{equation*}
Which can be simplified to:
\begin{align*}
& a^2 \|w + v\|^2 = (v \cdot \star[s]_2)^2 \frac{4 a^2}{l^2} \\
& a^2 (\|w + v\|^2 - 4 a^2 \|v\|^2) = 4 a^2 (v \cdot \star[s]_2)^2
\end{align*}
There are two cases. When $a \ne 0$ the following equation applies:
\begin{equation*}
\|w + v\|^2 = 4 (v \cdot \star[s]_2)^2 + 4 a^2 \|v\|^2
\end{equation*}
Using the above identity, there is:
\begin{align*}
l^2 \|w + v\|^2 & = \|w + v\|^2 - 4 a^2 \|v\|^2 = 4 (v \cdot \star[s]_2)^2 + 4 a^2 \|v\|^2 - 4 a^2 \|v\|^2 \\
l^2 \|w + v\|^2 & = 4 (v \cdot \star[s]_2)^2
\end{align*}
Modifying the identity \eqref{eqn:formula_ll} there is:
\begin{align*}
l^2 & = 1 - a^2 \frac{4 \|v\|^2}{\|w + v\|^2} \\
\|w + v\|^2 - 4 a^2 \|v\|^2 & = l^2 \|w + v\|^2 \\
\|w + v\|^2 - 4 a^2 \|v\|^2 & = 4 (v \cdot \star[s]_2)^2 \\
\|w + v\|^2 & = 4 (v \cdot \star[s]_2)^2 + 4 a^2 \|v\|^2 \\
\|w + v\| & = 2 \sqrt{(v \cdot \star[s]_2)^2 + a^2 \|v\|^2}
\end{align*}
Using the formula \eqref{eqn:formula_for_k} one obtains the value of $k$:
\begin{equation} \label{eqn:geneal_k_solution}
k = a \frac{2 \|v\|}{\|w+v\|} = a \frac{2 \|v\|}{2 \sqrt{(v \cdot \star[s]_2)^2 + a^2 \|v\|^2}} = \frac{a \|v\|}{\sqrt{(v \cdot \star[s]_2)^2 + a^2 \|v\|^2}}
\end{equation}
Again from \eqref{eqn:formula_ll} it is possible to calculate $l$:
\begin{align*}
l^2 & = 1 - a^2 \frac{4 \|v\|^2}{\|w + v\|^2} = 1 - \frac{4 a^2 \|v\|^2}{4 [ (v \cdot \star[s]_2)^2 + a^2 \|v\|^2]} \\
    & = \frac{(v \cdot \star[s]_2)^2 + a^2 \|v\|^2}{(v \cdot \star[s]_2)^2 + a^2 \|v\|^2} - \frac{a^2 \|v\|^2}{(v \cdot \star[s]_2)^2 + a^2 \|v\|^2} = \frac{(v \cdot \star[s]_2)^2}{(v \cdot \star[s]_2)^2 + a^2 \|v\|^2}
\end{align*}
It is possible to take the square root and abandon sign. This comes from the fact that when $s$ is a given spinor then the solution represented by variable $l$ must include the solution for $-s$ as well. In that case the numerator is negated, so also $-l$ is a solution automatically.
\begin{equation} \label{eqn:geneal_l_solution}
l = \frac{v \cdot \star[s]_2}{\sqrt{(v \cdot \star[s]_2)^2 + a^2 \|v\|^2}}
\end{equation}
In the case when $a = 0$ from \eqref{eqn:formula_ll} there is:
\begin{align*}
l^2 & = 1 \\
k^2 & = 1 - l^2 = 0
\end{align*}
The same argumentation as above can be used to select any solution of $l^2 = 1$. Both cases are special cases of \eqref{eqn:geneal_k_solution} and \eqref{eqn:geneal_l_solution} so it is sufficient to consider only those general solutions. The twist factor equal to:
\begin{align*}
q & = k + \e_{123} \N(v) l = \\
& = \frac{a \|v\|}{\sqrt{(v \cdot \star[s]_2)^2 + 4 a^2 \|v\|^2}} + \frac{ \e_{123} \N(v) (v \cdot \star[s]_2)}{\sqrt{(v \cdot \star[s]_2)^2 + a^2 \|v\|^2}} \\
& = \frac{a \|v\| + \N(v) (v \cdot \star[s])}{\sqrt{(v \cdot \star[s]_2)^2 + a^2 \|v\|^2}} = \frac{a \|v\| + \N(v) (v \cdot s) - \N(v) (v \cdot [s]_0)}{\sqrt{(v \cdot \star[s]_2)^2 + a^2 \|v\|^2}} \\
& = \frac{a \|v\| + \N(v) (v \cdot s) - a \|v\|}{\sqrt{(v \cdot \star[s]_2)^2 + a^2 \|v\|^2}} = \frac{\N(v) (v \cdot s)}{\sqrt{(v \cdot \star[s]_2)^2 + a^2 \|v\|^2}}
\end{align*}
Define a norm of a pinor $p = p_0 + p_{12} \e_{12} + p_{23} \e_{23} + p_{31} \e_{31} \in \Pin(3)$:
\begin{equation*}
\N(p) := \frac{p}{\sqrt{p \tilde{p}}} = \frac{p_0 + p_{12} \e_{12} + p_{23} \e_{23} + p_{31} \e_{31}}{\sqrt{p_0^2 + p_{12}^2 + p_{23}^2 + p_{31}^2}} \in \Spin(3)
\end{equation*}
where $\tilde{p} = p_0 - p_{12} \e_{12} - p_{23} \e_{23} - p_{31} \e_{31}$.

Using the fact that $s = a + \e_{123} \star [s]_2$  the numerator of twist factor can be written as:
\begin{align} \label{eqn:numerator_twist}
& \N(v) (v \cdot s) = \N(v) (v \cdot (a + \e_{123} \star [s]_2) = \N(v) (v a + \e_{123} v \cdot \star [s]_2) \notag \\
& = a \|v\| + \e_{123} \N(v) v \cdot \star [s]_2
\end{align}
the norm of the numerator of twist factor is equal to:
\begin{align*}
& \sqrt{(a \|v\| + \e_{123} \N(v) v \cdot \star [s]_2)(a \|v\| - \e_{123} \N(v) v \cdot \star [s]_2)} \\
& = \sqrt{a^2 \|v\|^2 + (\e_{123} \N(v) v \cdot \star [s]_2)^2} = \sqrt{a^2 \|v\|^2 + (v \cdot \star [s]_2)^2}
\end{align*}
one notices that it is equal to the denominator. In result, it is possible to write simply:
\begin{equation*}
q = \N(v (v \cdot s))
\end{equation*}
The swing factor is calculated by modifying the original equation:
\begin{align*}
s & = p q \\
p & = s q^{-1}
\end{align*}
by using \eqref{eqn:numerator_twist} twist inverse can be easily calculated:
\begin{align*}
& q^{-1} = \tilde{q} = \tilde{\sigma}_v(s) = \widetilde{a \|v\| + \e_{123} \N(v) v \cdot \star [s]_2} = a \|v\| - \e_{123} \N(v) v \cdot \star [s]_2 \\
& = a \|v\| + \e_{123} \N(v) v \cdot \star [-s]_2 = \N(v) (v a + \e_{123} v \cdot \star [-s]_2) \\
& = \N(v) (v \cdot (a + \e_{123} \star [-s]_2)) = \N(v) (v \cdot \tilde{s})
\end{align*}
so
\begin{equation*}
p = s \tilde{\sigma}_v(s)
\end{equation*}
where $\tilde{\sigma}_v(s) = \N(v) (v \cdot \tilde{s})$. Both swing and twist factors can be negated as a given spinor and its negation define the same rotation:
\begin{align*}
& p = \pm s \tilde{\sigma}_v(s) \\
& q = \pm \sigma_v(s)
\end{align*}
which completes the proof.
\end{proof}

The following dual theorem can be easily proved with the previous theorem \eqref{thm:swing_twist_decomposition_sat_type}:

\begin{theorem}[Swing-twist decomposition of a spinor in twist-after-swing representation] \label{thm:swing_twist_decomposition_tas_type}
Assume that $s \in \Spin(3)$ is a spinor. For any non-zero base vector $v \in \Cl_3$ such that $s v s^{-1} \ne -v$ there exist \textbf{a unique up to the sign} swing-twist decomposition in twist-after-swing representation
\begin{equation*}
s = \pm q p
\end{equation*}
where swing spinor $p$ and twist spinor $q$ are equal to:
\begin{align*}
& p = \pm \tilde{\sigma}_v(s) s \\
& q = \pm \sigma_v(s)
\end{align*}
where $\sigma_v(s)$ is a twist projection function.
\end{theorem}
\begin{proof}
Assume that $u = s^{-1}$. According to theorem \ref{thm:swing_twist_decomposition_tas_type} there exists a swing-twist decomposition of spinor $u$ in swing-after-twist representation:
\begin{equation*}
u = \pm p q
\end{equation*}
where $p$ is a swing factor and $q$ is a twist factor of $u$ in respect to the base vector $v$. Taking inverse of both sides one obtains:
\begin{equation*}
u^{-1} = s = \pm q^{-1} p^{-1}
\end{equation*}
which is a twist-after-swing decomposition of spinor $s$. Swing and twist factors can be rewritten as:
\begin{align*}
p^{-1} & = (u \tilde{\sigma}_v(u))^{-1} = (s^{-1} \tilde{\sigma}_v(s^{-1}))^{-1} = \tilde{\sigma}_v(s^{-1})^{-1} s = \tilde{\sigma}_v(s) s \\
q^{-1} & = \sigma_v(u)^{-1} = \sigma_v(s^{-1})^{-1} = \sigma_v(s)
\end{align*}
which completes the proof.
\end{proof}

\section{Twist projection function}

In this section basic properties of the twist projection function are discussed. Intuitively, twist projection function takes a spinor and returns its twist factor in respect to a given vector. In this paper the following definition was assumed
\begin{equation*}
\sigma_v(s) := \N(v (v \cdot s))
\end{equation*}
This is algebraically equivalent to projecting it onto a vector in the sense of Clifford algebra. This an improvement over the work of Huyghe \cite{huyghe} where the author introduces similar projection operator artificially in quaternion algebra. Twist projection function is a projection due to the following
\begin{proposition}
A twist projection function is a projection.
\end{proposition}
\begin{proof}
It is enough to prove that $\sigma_v(\sigma_v(s)) = \sigma_v(s)$ for any vector $v$ and any spinor $w$:
\begin{align*}
& \sigma_v(\sigma_v(s)) = \N(v (v \cdot \N(v (v \cdot s)))) = \N(v (v \cdot \frac{a \|v\| + \e_{123} \N(v) (v \cdot \star [s]_2)}{\sqrt{(v \cdot \star[s]_2)^2 + a^2 \|v\|^2}})) \\
& = \N(v \frac{v \cdot a \|v\| + \e_{123} (v \cdot \N(v)) (v \cdot \star [s]_2)}{\sqrt{(v \cdot \star[s]_2)^2 + a^2 \|v\|^2}}) = \N(v \frac{v \cdot a \|v\| + \e_{123} \|v\| v \cdot \star [s]_2}{\sqrt{(v \cdot \star[s]_2)^2 + a^2 \|v\|^2}}) \\
& = \N(v \frac{(\|v\| v) \cdot (a + \star [s]_2)}{\sqrt{(v \cdot \star[s]_2)^2 + a^2 \|v\|^2}}) = \N(\|v\|^2 \frac{\N(v) (v \cdot s)}{\sqrt{(v \cdot \star[s]_2)^2 + a^2 \|v\|^2}}) \\
& = \N(v (v \cdot s)) = \sigma_v(s)
\end{align*}
which completes the proof.
\end{proof}


%
%
%

\section{Applications}

In this section an exemplary application of the proposed theoretical results is presented. It is a fast, concise and numerically stable algorithm for calculating swing-twist decomposition of a spinor.

%
%
%

\subsection{An efficient algorithm for swing-twist decomposition of a spinor}

To propose an efficient method of calculating swing-twist decomposition the following formula is used for twist projection function:
\begin{align*}
& \sigma_v(s) = \N(v (v \cdot s)) = \N(v (v \cdot (a + \e_{123} \star [s]_2)) = \\
& = \N(v (v a + \e_{123} v \cdot \star [s]_2)) = \N(a \|v\|^2 + \e_{123} v (v \cdot \star [s]_2))
\end{align*}
If a rotation is represented by a pinor normalization can be omitted. In case of spinor representation normalization is needed and requires square root computation or equivalently requires arithmetic with square root extension. The proposed method uses swing-after-twist representation and is presented in algorithm \ref{algorithm_main}.
\begin{algorithm}
\caption{Compute swing-twist decomposition of a spinor $s = p q$}
\label{algorithm_main}
\begin{algorithmic}
\REQUIRE $v = x \e_1 + y \e_2 + z \e_3$, $s = a + b \e_{12} + c \e_{23} + d \e_{31}$
\ENSURE $\|v\| \neq 0$
\STATE $u \leftarrow x c + y d + z b$
\STATE $n \leftarrow x^2 + y^2 + z^2$
\STATE $m \leftarrow a n$
\STATE $l \leftarrow \sqrt{m^2 + u^2 n}$
\STATE $q \leftarrow \frac{m}{l} + \frac{z u}{l} \e_{12} + \frac{x u}{l} \e_{23} + \frac{y u}{l} \e_{31}$
\STATE $p \leftarrow s \tilde{q}$
\RETURN p, q
\end{algorithmic}
\end{algorithm}

Because of its simplicity the proposed method favourably compares to existing methods for computing swing-twist decomposition (as presented in section \ref{section_existing_and_related_solutions}.


\appendix

\section{Which decompositions are impossible}

When a combination of a spinor and a base vector cannot be decomposed into swing and twist. From theorem \ref{thm:swing_twist_decomposition_sat_type} or \ref{thm:swing_twist_decomposition_tas_type} a decomposition is impossible when:
\begin{equation*}
s v s^{-1} = -v
\end{equation*}
the above condition can be rewritten in terms of spinor coordinates. Using \eqref{eqn:spinor_rotation_formula} one writes
\begin{align} \label{eqn:impossible_decompositions_set}
\left\{
\begin{aligned}
& (a^2 - b^2 + c^2 - d^2 + 1) x + 2 y (a b + c d) + 2 z (b c - a d) = 0 \\
& (a^2 - b^2 - c^2 + d^2 + 1) y + 2 x (c d - a b) + 2 z (b d + a c) = 0 \\
& (a^2 + b^2 - c^2 - d^2 + 1) z + 2 x (b c + a d) + 2 y (b d - a c) = 0 \\
& a^2 + b^2 + c^2 + d^2 - 1 = 0
\end{aligned}
\right.
\end{align}
Denote as $W_i$ the left side of $i$th equation of \ref{eqn:impossible_decompositions_set}. A valid solution $(a, b, c, d)$ must satisfy the following equation:
\begin{equation*}
x W_1 + y W_2 + z W_3 + (x^2 + y^2 + z^2) W_4 = 0
\end{equation*}
the above formula can be expanded:
\begin{align*}
& (a^2 - b^2 + c^2 - d^2 + 1) x^2 + 2 x y (a b + c d) + 2 x z (b c - a d) \\
& + (a^2 - b^2 - c^2 + d^2 + 1) y^2 + 2 x y (c d - a b) + 2 y z (b d + a c) \\
& + (a^2 + b^2 - c^2 - d^2 + 1) z^2 + 2 x z (b c + a d) + 2 y z (b d - a c) \\
& + (x^2 + y^2 + z^2) (a^2 + b^2 + c^2 + d^2 - 1) = 0 \\
& (a^2 - b^2 + c^2 - d^2 + 1 + a^2 + b^2 + c^2 + d^2 - 1) x^2 + 2 x y c d + 2 x z b c \\
& + (a^2 - b^2 - c^2 + d^2 + 1 + a^2 + b^2 + c^2 + d^2 - 1) y^2 + 2 x y c d + 2 y z b d \\
& + (a^2 + b^2 - c^2 - d^2 + 1 + a^2 + b^2 + c^2 + d^2 - 1) z^2 + 2 x z b c + 2 y z b d = 0 \\
& (c x + d y + b z)^2 + a^2 (x^2 + y^2 + z^2) = 0
\end{align*}
which implies that the two identities must hold:
\begin{equation*}
a^2 (x^2 + y^2 + z^2) \quad \wedge \quad c x + d y + b z = 0
\end{equation*}
but since $\|v\| \ne 0$ it must be
\begin{equation*}
a = 0 \quad \wedge \quad v \cdot \star [s]_2 = 0
\end{equation*}
These conditions define when a spinor can be written as a swing-twist decomposition in respect to a given base vector.





\end{document}